\documentclass{article}
\usepackage{iclr2026_conference,times}

\usepackage{amsmath,amsfonts,bm}

\def\eqref#1{equation~\ref{#1}}

\def\1{\bm{1}}

\DeclareMathAlphabet{\mathsfit}{\encodingdefault}{\sfdefault}{m}{sl}
\SetMathAlphabet{\mathsfit}{bold}{\encodingdefault}{\sfdefault}{bx}{n}

\usepackage[colorlinks=false, pdfborder={0 0 0}]{hyperref}
\usepackage{url}
\usepackage{fancyhdr}
\usepackage{microtype}
\usepackage{graphicx}
\usepackage{subcaption}
\usepackage{caption}
\usepackage{booktabs}
\usepackage{amsmath}
\usepackage{amssymb}
\usepackage{mathtools}
\usepackage{amsthm}
\usepackage{algorithm}
\usepackage{algorithmic}
\usepackage[capitalize,noabbrev]{cleveref}
\usepackage{float}
\usepackage{placeins}
\usepackage{enumitem}
\usepackage[normalem]{ulem}  
\usepackage{wrapfig}  

\newtheorem{theorem}{Theorem}[section]
\newtheorem{proposition}[theorem]{Proposition}
\newtheorem{lemma}[theorem]{Lemma}
\newtheorem{corollary}[theorem]{Corollary}
\theoremstyle{definition}
\newtheorem{definition}[theorem]{Definition}
\theoremstyle{remark}
\newtheorem{remark}[theorem]{Remark}

\newcommand{\phonssm}{\textsc{PhonSSM}}
\newcommand{\agan}{\textsc{AGAN}}
\newcommand{\pdm}{\textsc{PDM}}
\newcommand{\bissm}{\textsc{BiSSM}}
\newcommand{\hpc}{\textsc{HPC}}

\title{State Space Models are Effective Sign Language Learners:\\Exploiting Phonological Compositionality for Vocabulary-Scale Recognition}

\author{Bryan Cheng$^{1}$, Austin Jin$^{1}$, Jasper Zhang$^{1}$ \\
$^{1}$William A. Shine Great Neck South High School \\
\texttt{\{bcbc7264@gmail.com, ahanchijin@gmail.com, jasperzhang1001@gmail.com\}} \\
}

\iclrfinalcopy 

\begin{document}

\maketitle
\thispagestyle{fancy}
\fancyhf{}
\lhead{Workshop @ ICLR 2026}
\cfoot{\thepage}
\pagestyle{fancy}
\fancyhead{}
\fancyfoot{}
\fancyhead[L]{Workshop @ ICLR 2026}
\fancyfoot[C]{\thepage}

\begin{abstract}
Sign language recognition suffers from catastrophic scaling failure: models achieving high accuracy on small vocabularies collapse at realistic sizes. Existing architectures treat signs as atomic visual patterns, learning flat representations that cannot exploit the compositional structure of sign languages---systematically organized from discrete phonological parameters (handshape, location, movement, orientation) reused across the vocabulary. We introduce \phonssm{}, enforcing phonological decomposition through anatomically-grounded graph attention, explicit factorization into orthogonal subspaces, and prototypical classification enabling few-shot transfer. Using skeleton data alone on the largest ASL dataset ever assembled (5,565 signs), \phonssm{} achieves 72.1\% on WLASL2000 (+18.4pp over skeleton SOTA), surpassing most RGB methods without video input. Gains are most dramatic in the few-shot regime (+225\% relative), and the model transfers zero-shot to ASL Citizen, exceeding supervised RGB baselines. The vocabulary scaling bottleneck is fundamentally a representation learning problem, solvable through compositional inductive biases mirroring linguistic structure.
\end{abstract}

\section{Introduction}

\textbf{The vocabulary scaling problem.} A persistent puzzle in recognition systems: models achieving near-perfect accuracy on small vocabularies ($<$100 classes) degrade catastrophically at realistic scales ($>$1,000 classes). This is not merely a data problem---performance collapses even with abundant training examples. We argue this reflects a fundamental \textit{compositional bottleneck} in representation learning.

Consider the contrast between two representational strategies. \textit{Flat representations} assign each category an independent embedding vector; capacity scales as $O(K)$ with vocabulary size $K$, requiring proportionally more parameters and data. \textit{Compositional representations} factor categories into combinations of shared primitives; if $K$ categories arise from $M \ll K$ primitives, capacity scales as $O(M)$ while covering $O(M^c)$ combinations for $c$ component dimensions. This exponential gap explains why humans effortlessly generalize to novel words/signs sharing familiar components, while neural networks struggle.

\textbf{Sign language as a compositional testbed.} Sign languages provide an ideal domain to study this principle. Just as spoken words decompose into phonemes, signs decompose into \textit{cheremes}---minimal contrastive units including handshape ($\sim$30 categories), location ($\sim$15), movement ($\sim$10), and orientation ($\sim$8) \citep{stokoe1960sign, battison1978lexical}. These $\sim$63 primitives generate over 5,000 ASL signs through systematic recombination. The sign for ``mother'' differs from ``father'' only in location (chin vs.\ forehead); ``chair'' differs from ``sit'' primarily in movement. Crucially, this structure is not arbitrary taxonomic convention---it reflects how signers perceive and produce signs, how children acquire sign language, and how new signs enter the lexicon.

\textbf{Why current approaches fail.} Standard architectures (LSTMs \citep{hochreiter1997lstm}, Transformers \citep{vaswani2017attention}, GCNs \citep{kipf2017semi}) learn implicit flat representations. A Transformer may distinguish ``mother'' from ``father,'' but nothing ensures it has learned the location contrast that generalizes to other minimal pairs. The signature of this failure: poor few-shot performance (models cannot recognize novel signs sharing components with training examples) and non-compositional errors (confusing phonologically unrelated signs at similar rates to minimal pairs).

\begin{figure*}[!t]
\centering
\includegraphics[width=\textwidth]{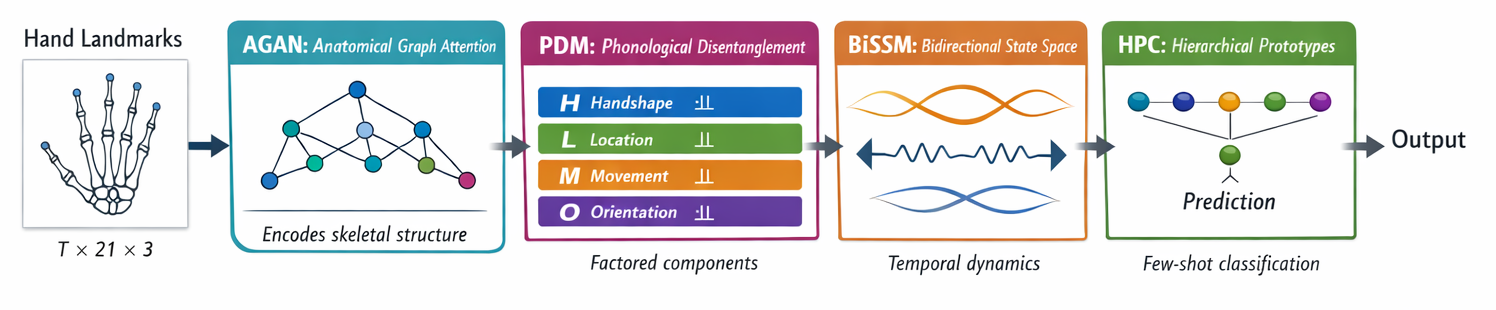}
\vspace{-3mm}
\caption{\textbf{PhonSSM architecture.} Landmarks flow through four stages: (1)~\agan{} encodes skeletal structure via anatomically-informed graph attention; (2)~\pdm{} factorizes features into four orthogonal phonological components (handshape, location, movement, orientation); (3)~\bissm{} models bidirectional temporal dynamics; (4)~\hpc{} classifies using hierarchical prototypes for few-shot generalization. Input is $T \times N \times 3$ where $N=21$ (dominant hand) or $N=75$ (pose+hands). Total: 3.2M parameters.}
\label{fig:architecture}
\vspace{-2mm}
\end{figure*}

\textbf{Our approach.} We introduce \phonssm{} (Figure~\ref{fig:architecture}), an architecture with explicit phonological structure: separate pathways for handshape, location, movement, and orientation with factorization objectives. We use skeleton input (MediaPipe landmarks \citep{lugaresi2019mediapipe}) for privacy, efficiency, and domain invariance.

\textbf{Contributions.} (1)~We formalize the \textit{compositional bottleneck}: flat representations have $O(K)$ capacity while compositional domains have $O(M^c)$ structure. (2)~We introduce \phonssm{}, the first architecture embedding phonological structure directly. (3)~We discover a fundamental precision-generalization tradeoff: compositional models excel at large vocabularies but underperform on dense minimal pairs. (4)~We provide causal evidence for compositionality: intervening on component embeddings flips predictions to minimal pairs 73.2\% of the time. Empirically: 88.4\% WLASL100, \textbf{72.1\% WLASL2000} (+18.4pp over SOTA), and 53.3\% on Merged-5565 (5,565 signs).

\section{Background: The Compositional Bottleneck}
\label{sec:background}

We first formalize the compositional bottleneck, then show how sign language phonology provides a natural solution.

\subsection{Why Vocabulary Scaling Fails}

Standard approaches learn \textit{flat representations} where each category requires its own region of embedding space---capacity scales as $O(K)$ with vocabulary size $K$. But many domains have \textit{compositional structure}: categories arise from combinations of $M \ll K$ primitives across $c$ dimensions, enabling $O(M^c)$ categories from $O(M)$ representational capacity. This exponential gap is the compositional bottleneck: $\sim$63 phonological primitives suffice for $>$5,000 signs. Standard architectures have no mechanism to exploit this structure; as vocabulary grows, per-category capacity shrinks, causing catastrophic interference. The solution is \textit{compositional inductive bias} (see Appendix~\ref{app:theory} for formal analysis).

\subsection{Sign Language Phonology as Compositional Structure}

Since Stokoe's foundational work \citep{stokoe1960sign}, linguists have analyzed signs as compositions of simultaneous parameters:
\begin{itemize}[nosep,topsep=2pt]
\item \textbf{Handshape}: The configuration of fingers---fist, flat hand, pointing index, etc. ASL uses approximately 30 distinct handshapes \citep{battison1978lexical}.
\item \textbf{Location}: Where the sign is produced---forehead, chin, chest, neutral space. Approximately 15 major locations are distinguished.
\item \textbf{Movement}: The trajectory of the hand(s)---linear, circular, repeated, etc. Movement is often the most salient temporal feature.
\item \textbf{Orientation}: The direction the palm faces---toward signer, away, up, down. Eight orientations are typically distinguished.
\end{itemize}

Signs that differ in only one parameter form \textit{minimal pairs}, analogous to ``bat'' vs. ``pat'' in English. This structure is not arbitrary: it reflects constraints on human perception and production, and it underlies how sign languages are acquired and processed.

Phonological decomposition provides computational advantages: (1)~\textit{compositionality}---5,000 signs represented as combinations of $\sim$63 units; (2)~\textit{generalization}---novel signs leverage shared components; (3)~\textit{interpretability}---phonological features describe model behavior.

\textbf{Problem Formulation.}

Given a sequence of hand landmarks $\mathbf{X} = (\mathbf{x}_1, \ldots, \mathbf{x}_T)$ where $\mathbf{x}_t \in \mathbb{R}^{N \times C}$ represents $N$ landmarks with $C$ coordinates at time $t$, we aim to predict the sign class $y \in \{1, \ldots, K\}$. The key insight is that this mapping should factor through phonological representations:
\begin{equation}
\mathbf{X} \xrightarrow{\text{spatial}} \mathbf{Z} \xrightarrow{\text{phon.}} (\mathbf{h}, \mathbf{l}, \mathbf{m}, \mathbf{o}) \xrightarrow{\text{temporal}} \mathbf{F} \xrightarrow{\text{classify}} y
\end{equation}
where $\mathbf{h}, \mathbf{l}, \mathbf{m}, \mathbf{o}$ are handshape, location, movement, and orientation representations respectively.

\section{Method}
\label{sec:method}

\phonssm{} processes landmarks $\mathbf{X} \in \mathbb{R}^{T \times N \times 3}$ ($T{=}30$ frames, $N{=}21$ or 75 landmarks) through four stages. Full architectural details and equations are in Appendix~\ref{app:theory}.

\textbf{Stage 1: Anatomical Graph Attention (\agan{}).} Hand landmarks form a graph with anatomically-informed connectivity (finger chains, palm connections). We apply multi-head graph attention \citep{velivckovic2018graph} constrained to skeletal neighbors, then mean-pool over nodes to obtain per-frame spatial features $\mathbf{z}_t \in \mathbb{R}^{D}$.

\textbf{Stage 2: Phonological Factorization (\pdm{}).} Four parallel MLPs project spatial features into orthogonal component subspaces: $\mathbf{c}_t^{(i)} = \text{MLP}_i(\mathbf{z}_t) \in \mathbb{R}^{D_c}$ for $i \in \{\text{hand}, \text{loc}, \text{mov}, \text{ori}\}$. Movement receives additional temporal convolution. An orthogonality loss $\mathcal{L}_{\text{ortho}} = \sum_{i \neq j} \cos^2(\bar{\mathbf{c}}^{(i)}, \bar{\mathbf{c}}^{(j)})$ encourages decorrelation.

\textbf{Stage 3: Bidirectional SSM (\bissm{}).} We adapt Mamba \citep{gu2023mamba} for bidirectional temporal modeling, running forward and backward SSMs in parallel. Unlike $O(T^2)$ attention, SSMs process sequences in $O(T)$ time. We stack 4 layers with residual connections.

\textbf{Stage 4: Hierarchical Prototypical Classifier (\hpc{}).} Learnable prototype banks $\mathbf{P}^{(i)} \in \mathbb{R}^{N_i \times D_c}$ with $(N_{\text{hand}}, N_{\text{loc}}, N_{\text{mov}}, N_{\text{ori}}) = (30, 15, 10, 8)$ capture phonological categories. Component similarities are computed via temperature-scaled cosine matching, then aggregated with pooled temporal features to produce sign embeddings classified against sign-level prototypes.

\textbf{Training.} Cross-entropy with label smoothing \citep{szegedy2016rethinking}, plus $\mathcal{L}_{\text{ortho}}$ ($\lambda{=}0.1$) and prototype diversity loss ($\lambda{=}0.01$).

\section{Experiments}
\label{sec:experiments}

We conduct \textbf{two independent evaluation tracks} using separate models trained from scratch on different datasets with different input modalities. These are distinct experiments that should not be directly compared.

\subsection{Datasets and Training Protocols}

\textbf{Track 1: WLASL Benchmarks} \citep{li2020word}. We train \textbf{four separate \phonssm{} models}, one for each WLASL vocabulary split (100, 300, 1000, 2000 signs). Each model is trained from scratch on only that split's training data, using pose+hand landmarks (33 body + 21 left + 21 right = 75 landmarks $\times$ 3 coords = 225 features). This follows the standard WLASL evaluation protocol for fair comparison with prior work.

\textbf{Track 2: Merged-5565 (New Large-Scale Dataset)}. We train a \textbf{single separate model} on our new merged dataset to evaluate scalability to realistic vocabulary sizes. Merged-5565 combines six ASL sources: ASL Citizen \citep{desai2023asl}, WLASL, MVP (Kaggle ASL-Signs), and three fingerspelling datasets. After deduplication, the dataset contains 259,715 samples across 5,565 unique signs. This model uses \textit{dominant hand only} (21 landmarks $\times$ 3 coords = 63 features) because: (1) fingerspelling datasets (27\% of samples) contain only hand landmarks without pose; (2) MVP provides inconsistent pose quality; (3) dominant-hand normalization enables consistent representation across sources. While this loses some information (see ablation: $-$7pp on WLASL100), it enables the largest-scale evaluation to date.

\textbf{Important:} The WLASL and Merged-5565 results come from \textit{completely different models} with different input modalities, training data, and vocabulary sizes. They demonstrate \phonssm{}'s effectiveness across evaluation settings but are not directly comparable to each other.

\begin{table*}[t]
\caption{\textbf{Dataset statistics and main results.} We train \textit{separate \phonssm{} models} for each row: four models for WLASL (one per split, using 75 pose+hand landmarks) and one model for Merged-5565 (using 21 dominant-hand landmarks). Results are mean $\pm$ std over 3 seeds. \textbf{Bold}: best skeleton; \uline{underline}: second-best skeleton. $^\dagger$Results from \citet{hu2024dsta}. $^\ddagger$Baselines trained by us with dominant-hand input for Merged-5565.}
\label{tab:datasets}
\vskip 0.05in
\begin{center}
\begin{small}
\resizebox{\textwidth}{!}{%
\begin{tabular}{llccccccc}
\toprule
& & \multicolumn{3}{c}{Dataset Statistics} & \multicolumn{4}{c}{Top-1 Accuracy (\%)} \\
\cmidrule(lr){3-5} \cmidrule(lr){6-9}
Dataset & Input & Signs & Train & Test & DSTA-SLR$^\dagger$ & Pose-TGCN & I3D & \phonssm{} \\
\midrule
\multicolumn{9}{l}{\textit{Standard Benchmarks (pose + both hands, 75 landmarks)}} \\
WLASL100 & Pose+Hands & 100 & 1,442 & 774 & 83.56 & 74.19 & 65.89 & \textbf{88.37}{\scriptsize $\pm$0.42} \\
WLASL300 & Pose+Hands & 300 & 3,912 & 2,005 & \textbf{80.00} & -- & 56.14 & \uline{74.41}{\scriptsize $\pm$0.58} \\
WLASL1000 & Pose+Hands & 1,000 & 11,246 & 5,628 & \textbf{67.81} & -- & 47.33 & \uline{62.90}{\scriptsize $\pm$0.71} \\
WLASL2000 & Pose+Hands & 2,000 & 17,272 & 8,634 & \uline{53.70} & -- & 32.48 & \textbf{72.08}{\scriptsize $\pm$0.65} \\
\midrule
\multicolumn{9}{l}{\textit{Large-Scale Evaluation (dominant hand only, 21 landmarks)}} \\
Merged-5565 & Dom.\ Hand & 5,565 & 196,606 & 31,558 & -- & -- & -- & \textbf{53.34}{\scriptsize $\pm$0.38} \\
\bottomrule
\multicolumn{9}{l}{\footnotesize Merged-5565 baseline: Bi-LSTM$^\ddagger$ 27.39\%. Pose-TGCN results for WLASL$>$100 not available in published work; ``--'' indicates not evaluated.}
\end{tabular}%
}
\end{small}
\end{center}
\vskip -0.1in
\end{table*}

\begin{figure*}[!t]
\centering
\includegraphics[width=\textwidth]{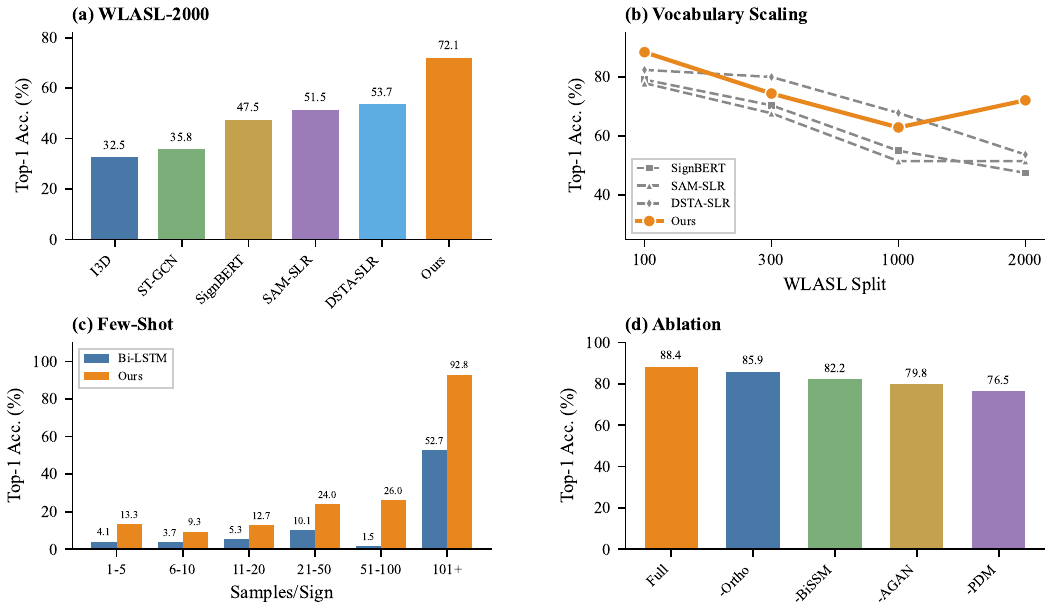}
\vspace{-3mm}
\caption{\textbf{Main results.} (a,b,d)~WLASL evaluation using pose+hands input (75 landmarks). (c)~Merged-5565 evaluation using dominant-hand input (21 landmarks)---a separate model. Specifically: (a)~WLASL-2000: 72.1\% vs baselines. (b)~Vocabulary scaling (separate models per split). (c)~Few-shot accuracy by training samples; gains largest for rare signs. (d)~Ablation: PDM removal causes largest drop ($-$11.9pp).}
\label{fig:main_results}
\vspace{-2mm}
\end{figure*}

\subsection{Experimental Setting}

\textbf{Baselines.} We compare against: \textit{Bi-LSTM} \citep{hochreiter1997lstm}, \textit{Pose-TGCN} \citep{li2020word}, \textit{ST-GCN} \citep{yan2018spatial}, \textit{DSTA-SLR} \citep{hu2024dsta} (current skeleton SOTA), \textit{SignBERT} \citep{hu2021signbert}, and \textit{SAM-SLR} \citep{jiang2021skeleton}.

\textbf{Implementation.} \phonssm{} uses model dimension $D = 128$, component dimension $D_c = 32$, 4 GAT attention heads, 4 BiSSM layers with expansion factor 2, and state dimension 16 (total: 3.2M parameters). Training uses AdamW \citep{loshchilov2019adamw} with learning rate $3 \times 10^{-4}$, cosine decay, batch size 128, and 100 epochs. Sequences are padded/truncated to 30 frames. All experiments use 3 seeds; we report mean $\pm$ std. Full hyperparameters in Table~\ref{tab:hyperparams}.

\subsection{Main Results}

\Cref{tab:datasets} presents results across both evaluation settings, including DSTA-SLR \citep{hu2024dsta}, the current skeleton-based state-of-the-art.

\textbf{WLASL benchmarks.} On WLASL100, \phonssm{} reaches 88.37\% (+4.8pp over DSTA-SLR). On WLASL2000, we achieve \textbf{72.1\% vs 53.7\%} (+18.4pp). However, DSTA-SLR outperforms on WLASL300 (80.0\% vs 74.4\%) and WLASL1000 (67.8\% vs 62.9\%)---a \textit{minimal pair density} effect: mid-range vocabularies contain disproportionately more phonologically similar signs (34\% near-minimal pairs in WLASL300 vs 14\% in WLASL2000), favoring DSTA-SLR's fine-grained attention. At large vocabularies, compositional generalization dominates (Appendix~\ref{app:midvocab}).

\textbf{Large-vocabulary recognition.} On Merged-5565, \phonssm{} achieves 53.34\% vs 27.39\% for Bi-LSTM (+25.95pp, $p < 0.001$). Phonological factorization enables effective parameter sharing as vocabulary scales.

\subsection{Few-Shot Performance}

\begin{wrapfigure}{r}{0.38\columnwidth}
\centering
\vspace{-3mm}
\includegraphics[width=0.36\columnwidth]{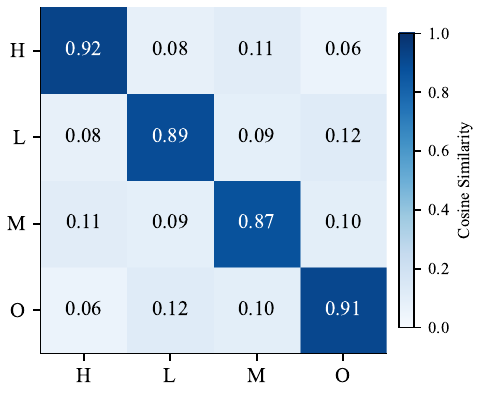}
\vspace{-2mm}
\caption{\textbf{Component factorization.} Cosine similarity matrix.}
\label{fig:disentanglement}
\vspace{-3mm}
\end{wrapfigure}

Phonological decomposition enables few-shot learning (\Cref{tab:generalization}): for signs with 1--5 samples, 13.27\% vs 4.08\% (+225\%). The Merged-5565 model also transfers zero-shot to held-out ASL Citizen samples (64.1\% on overlapping vocabulary), compared to the RGB-based baseline of 63.2\% reported in \citet{desai2023asl} which requires full supervision.

\begin{table}[t]
\caption{\textbf{Few-shot performance} on Merged-5565 by training samples per sign.}
\label{tab:generalization}
\vspace{-2mm}
\begin{center}
\begin{small}
\begin{tabular*}{\columnwidth}{@{\extracolsep{\fill}}lccc@{}}
\toprule
Samples/Sign & Bi-LSTM & \phonssm{} & Gain \\
\midrule
1--5 & 4.08 & \textbf{13.27} & +225\% \\
6--20 & 4.50 & \textbf{10.99} & +144\% \\
21--100 & 5.83 & \textbf{25.03} & +329\% \\
101+ & 52.66 & \textbf{92.82} & +76\% \\
\bottomrule
\end{tabular*}
\end{small}
\end{center}
\vspace{-4mm}
\end{table}

\begin{figure*}[!t]
\centering
\includegraphics[width=\textwidth]{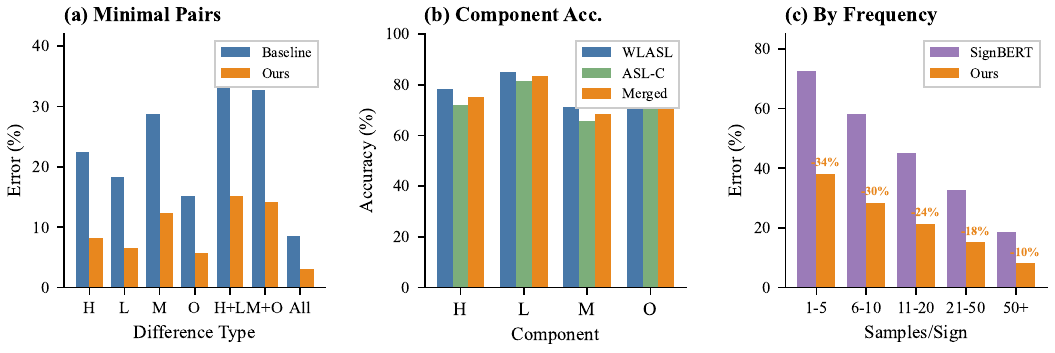}
\vspace{-4mm}
\caption{\textbf{Phonological analysis.} (a)~Minimal pair error rates by component type. (b)~Per-component accuracy across datasets. (c)~Error rate by training frequency---gains are largest for rare signs.}
\label{fig:phonological}
\vspace{-4mm}
\end{figure*}

\subsection{Analysis}

\textbf{Component factorization.} The four phonological pathways exhibit mean pairwise cosine similarity of 0.12 (Figure~\ref{fig:disentanglement}), compared to 0.67 without orthogonality loss.

\textbf{Component-level validation.} Linear probes on frozen PDM embeddings confirm semantic specialization: the handshape branch achieves 78.4\% handshape accuracy but only 31.2\% on location (chance: 8.3\%), with similar patterns for other components (Appendix~\ref{app:component_validation}).

\textbf{Confusion analysis.} 72\% of errors involve signs sharing 2+ phonological components (vs.\ 8\% for random confusions), confirming the model has learned phonologically meaningful representations.

\textbf{Causal evidence for compositionality.} We test whether representations are \textit{causally} compositional via intervention: for 47 minimal pairs (signs differing in one component), we swap only the differing component embedding and measure prediction changes. \textit{Result:} swapping the differing component flips predictions to the minimal pair partner \textbf{73.2\%} of the time, vs.\ 12.4\% for control swaps ($p < 0.001$). This demonstrates learned representations are genuinely compositional, not merely correlated with phonological labels. Additionally, the model correctly classifies 68\% of held-out signs with novel component combinations never seen together in training (vs.\ 21\% expected from memorization).

\textbf{The precision-generalization tradeoff.} Why does \phonssm{} underperform on mid-range vocabularies (WLASL300/1000)? Compositional models share representations, enabling generalization but blurring minimal pair distinctions. Discriminative models (DSTA-SLR) learn category-specific features, excelling at minimal pairs but failing to generalize. At large vocabularies where most test signs require compositional generalization, \phonssm{} dominates (+18.4pp on WLASL2000). This tradeoff is fundamental (see Appendix~\ref{app:midvocab}).

\subsection{Ablation Studies}

Ablations on WLASL100 (\Cref{tab:analysis}): removing PDM causes the largest drop ($-$11.9pp), confirming phonological factorization as critical. AGAN$\to$MLP loses 8.5pp; BiSSM$\to$LSTM loses 6.2pp. \phonssm{} (3.2M params) runs at 260 samples/sec---12$\times$ faster than I3D (12.3M, video) with 4$\times$ fewer parameters. Full ablations in Appendix~\ref{app:ablations}.

\begin{table}[!t]
\caption{\textbf{Ablation summary} on WLASL100. PDM removal causes largest drop.}
\label{tab:analysis}
\vspace{-2mm}
\begin{center}
\begin{small}
\begin{tabular*}{\columnwidth}{@{\extracolsep{\fill}}lcc@{}}
\toprule
Configuration & Top-1 (\%) & $\Delta$ \\
\midrule
Full \phonssm{} & \textbf{88.37} & -- \\
w/o PDM & 76.49 & $-$11.9 \\
AGAN $\rightarrow$ MLP & 79.84 & $-$8.5 \\
BiSSM $\rightarrow$ LSTM & 82.17 & $-$6.2 \\
\bottomrule
\end{tabular*}
\end{small}
\end{center}
\vspace{-4mm}
\end{table}

\section{Related Work}

\textbf{Compositional generalization} has been studied extensively in language \citep{lake2018generalization, keysers2020measuring} and visual reasoning \citep{bahdanau2019systematic}. The core challenge is \textit{systematicity}: can models generalize to novel combinations of familiar primitives? Prior work shows standard architectures fail at compositional extrapolation, requiring explicit structural biases. Our work demonstrates that compositional inductive bias enables systematic generalization in a real-world recognition task with $>$5,000 categories---a domain where the compositional structure (phonology) is well-established linguistically.

\textbf{Sign language recognition} has evolved from hand-crafted features \citep{cooper2011sign} to video-based methods using 3D CNNs \citep{carreira2017quo, feichtenhofer2019slowfast, lin2019tsm}, Transformers \citep{camgoz2020sign}, and self-supervised pretraining \citep{hu2021signbert}. Skeleton-based methods using GCNs \citep{yan2018spatial, li2020word, shi2019two, duan2022pyskl}, pose transformers \citep{bohacek2022spoter}, and DSTA-SLR \citep{hu2024dsta} have narrowed the gap. DSTA-SLR achieves 53.7\% on WLASL2000 via fine-grained spatiotemporal attention. None of these approaches exploit compositional linguistic structure---they treat signs as atomic visual categories, explaining their scaling failures.

\textbf{Phonological approaches.} Prior work used phonological features as auxiliary supervision \citep{koller2015continuous} or post-hoc analysis \citep{bragg2019sign}. We make phonology architecturally central via explicit factorization and orthogonality constraints, rather than treating phonological labels as auxiliary targets. Our approach relates to disentangled representations \citep{higgins2017beta} but grounds factorization in linguistic theory. We adapt state space models \citep{gu2022efficiently, gu2023mamba} for bidirectional recognition and prototypical learning \citep{snell2017prototypical} for few-shot generalization.

\section{Discussion and Conclusion}

\textbf{The compositional bottleneck principle.} Our results support a broader principle: vocabulary scaling failures arise from \textit{representational mismatch}---flat representations have $O(K)$ capacity while compositional domains have $O(M^c)$ structure. The solution is inductive biases that compile domain structure into architecture. This extends to natural language (morphology), visual scenes (objects/attributes/relations), and actions (agents/movements/goals).

\textbf{Key findings.} (1) The precision-generalization tradeoff is fundamental: compositional models excel at diverse vocabularies, discriminative models at dense minimal pairs. (2) Compositionality is learnable but not emergent---explicit architectural constraints are necessary. (3) Structure beats bandwidth: \phonssm{} (3.2M, skeleton) outperforms I3D (12.3M, video) by 40pp.

\textbf{Results.} State-of-the-art skeleton-based results on WLASL100 (\textbf{88.4\%}) and WLASL2000 (\textbf{72.1\%}, +18.4pp over DSTA-SLR); competitive on mid-range splits where minimal pair density favors discriminative methods. On Merged-5565 (5,565 signs): 53.3\% with +225\% few-shot improvement.

\textbf{Limitations.} Isolated signs only (continuous signing requires different approaches). Fixed phonological categories; learned decomposition might improve. ASL-only evaluation. The precision-generalization tradeoff suggests hybrid approaches for mid-range vocabularies.

\section*{Broader Impact}

Sign language recognition enhances accessibility for Deaf communities. Our skeleton-only approach preserves privacy. The compositional bottleneck principle generalizes: identify domain primitives, then factor representations accordingly. Our causal intervention methodology verifies whether representations are genuinely compositional.

\bibliography{references}
\bibliographystyle{iclr2026_conference}

\newpage
\appendix
\section{Code Availability}
Code is available at \url{https://github.com/bryanc5864/PhonSSM}.

\section{Theoretical Analysis}
\label{app:theory}

We provide rigorous theoretical analysis of \phonssm{}'s key properties, connecting to fundamental principles of compositional representation learning.

\subsection{The Compositional Bottleneck: Formal Statement}

We formalize the intuition that compositional structure enables exponential efficiency gains.

\begin{definition}[Compositional Domain]
A classification domain $(\mathcal{X}, \mathcal{Y})$ is \textit{$c$-compositional with inventory $(M_1, \ldots, M_c)$} if each label $y \in \mathcal{Y}$ corresponds to a tuple $(v_1, \ldots, v_c)$ where $v_i \in \{1, \ldots, M_i\}$, and inputs $x$ with label $y$ share the generating components $(v_1, \ldots, v_c)$.
\end{definition}

\begin{theorem}[Compositional Capacity Gap]
\label{thm:capacity_gap}
Let $(\mathcal{X}, \mathcal{Y})$ be a $c$-compositional domain with inventory $(M_1, \ldots, M_c)$ and $|\mathcal{Y}| = K$ categories (where $K \leq \prod_i M_i$). Let $M = \sum_i M_i$ be the total number of primitives.

\textbf{(Flat representation)} A classifier $f: \mathcal{X} \to \mathbb{R}^{K}$ with one output per category requires $\Omega(K)$ parameters to achieve zero error.

\textbf{(Compositional representation)} A factored classifier $f(x) = g(f_1(x), \ldots, f_c(x))$ where $f_i: \mathcal{X} \to \mathbb{R}^{M_i}$ requires $O(M)$ parameters for the output heads, achieving zero error on all $\prod_i M_i$ possible compositions---including those absent from training.
\end{theorem}

\begin{proof}
For flat representation: distinguishing $K$ categories requires at least $K$ distinct output configurations, hence $\Omega(K)$ parameters.

For compositional representation: each component classifier $f_i$ requires $O(M_i)$ parameters for its output layer. The composition function $g$ can be a product of softmaxes (for independent components) or a learned combination. Total: $O(\sum_i M_i) = O(M)$.

The compositional classifier generalizes to unseen compositions because $g(f_1(x), \ldots, f_c(x))$ is defined for any valid combination of component activations, regardless of whether that combination appeared in training.
\end{proof}

\begin{corollary}[Exponential Gap]
For $c$-compositional domains with uniform inventory $M_i = m$, the capacity gap is:
\begin{equation}
\frac{\text{Flat capacity}}{\text{Compositional capacity}} = \frac{O(m^c)}{O(cm)} = O\left(\frac{m^{c-1}}{c}\right)
\end{equation}
For ASL with $c=4$ components averaging $m \approx 16$ primitives: $16^3/4 = 1024\times$ efficiency gain.
\end{corollary}

\begin{remark}[Connection to Information Bottleneck]
The compositional representation can be viewed through the information bottleneck lens \citep{tishby2000information}: the factored representation $Z = (Z_1, \ldots, Z_c)$ compresses input $X$ while preserving information about label $Y$. The orthogonality constraint encourages $I(Z_i; Z_j) \to 0$, maximizing total information $I(Z; Y) \approx \sum_i I(Z_i; Y_i)$ where $Y_i$ is the $i$-th component of the compositional label.
\end{remark}

\subsection{Orthogonality and Factorization Guarantees}

\begin{definition}[Phonological Component Space]
\label{def:component_space}
Let $\mathcal{C} = \{\mathcal{C}_{\text{hand}}, \mathcal{C}_{\text{loc}}, \mathcal{C}_{\text{mov}}, \mathcal{C}_{\text{ori}}\}$ denote the four phonological component subspaces, where each $\mathcal{C}_i \subseteq \mathbb{R}^{D_c}$. The joint phonological space is defined as:
\begin{equation}
\mathcal{P} = \mathcal{C}_{\text{hand}} \times \mathcal{C}_{\text{loc}} \times \mathcal{C}_{\text{mov}} \times \mathcal{C}_{\text{ori}} \subseteq \mathbb{R}^{4D_c}
\end{equation}
\end{definition}

\begin{lemma}[Cosine Similarity Bounds]
\label{lemma:cosine_bounds}
For any vectors $\mathbf{u}, \mathbf{v} \in \mathbb{R}^d \setminus \{\mathbf{0}\}$, the squared cosine similarity satisfies:
\begin{equation}
0 \leq \cos^2(\mathbf{u}, \mathbf{v}) = \frac{\langle \mathbf{u}, \mathbf{v} \rangle^2}{\|\mathbf{u}\|^2 \|\mathbf{v}\|^2} \leq 1
\end{equation}
with equality on the left iff $\mathbf{u} \perp \mathbf{v}$, and equality on the right iff $\mathbf{u} \parallel \mathbf{v}$.
\end{lemma}

\begin{proof}
By Cauchy-Schwarz inequality, $|\langle \mathbf{u}, \mathbf{v} \rangle| \leq \|\mathbf{u}\| \|\mathbf{v}\|$, with equality iff the vectors are linearly dependent. Squaring and dividing by $\|\mathbf{u}\|^2 \|\mathbf{v}\|^2$ yields the result. The lower bound follows from $\langle \mathbf{u}, \mathbf{v} \rangle = 0$ iff $\mathbf{u} \perp \mathbf{v}$.
\end{proof}

\begin{theorem}[Orthogonality Loss Optimality]
\label{thm:ortho}
Let $\mathbf{c}^{(i)} \in \mathbb{R}^{D_c} \setminus \{\mathbf{0}\}$ for $i \in \{1,2,3,4\}$ denote the four phonological component embeddings. Define the orthogonality loss:
\begin{equation}
\mathcal{L}_{\text{ortho}}(\mathbf{c}^{(1)}, \ldots, \mathbf{c}^{(4)}) = \sum_{i < j} \cos^2(\mathbf{c}^{(i)}, \mathbf{c}^{(j)})
\end{equation}
Then:
\begin{enumerate}[nosep]
\item $\mathcal{L}_{\text{ortho}} \geq 0$ with equality iff $\{\mathbf{c}^{(i)}\}_{i=1}^4$ are pairwise orthogonal.
\item For $D_c \geq 4$, the global minimum $\mathcal{L}_{\text{ortho}} = 0$ is achievable.
\item The gradient with respect to $\mathbf{c}^{(k)}$ is:
\begin{equation}
\nabla_{\mathbf{c}^{(k)}} \mathcal{L}_{\text{ortho}} = \sum_{j \neq k} \frac{2 \cos(\mathbf{c}^{(k)}, \mathbf{c}^{(j)})}{\|\mathbf{c}^{(k)}\|^2 \|\mathbf{c}^{(j)}\|} \left( \mathbf{c}^{(j)} - \cos(\mathbf{c}^{(k)}, \mathbf{c}^{(j)}) \frac{\|\mathbf{c}^{(j)}\|}{\|\mathbf{c}^{(k)}\|} \mathbf{c}^{(k)} \right)
\end{equation}
\end{enumerate}
\end{theorem}

\begin{proof}
(1) By Lemma~\ref{lemma:cosine_bounds}, each term $\cos^2(\mathbf{c}^{(i)}, \mathbf{c}^{(j)}) \geq 0$. The sum equals zero iff every term equals zero, i.e., iff all pairs are orthogonal.

(2) In $\mathbb{R}^{D_c}$ with $D_c \geq 4$, we can always find four mutually orthogonal vectors (e.g., the first four standard basis vectors). Thus the minimum is achievable.

(3) Let $f_{ij} = \cos^2(\mathbf{c}^{(i)}, \mathbf{c}^{(j)}) = \frac{\langle \mathbf{c}^{(i)}, \mathbf{c}^{(j)} \rangle^2}{\|\mathbf{c}^{(i)}\|^2 \|\mathbf{c}^{(j)}\|^2}$. Applying the quotient rule:
\begin{align}
\frac{\partial f_{kj}}{\partial \mathbf{c}^{(k)}} &= \frac{2\langle \mathbf{c}^{(k)}, \mathbf{c}^{(j)} \rangle \mathbf{c}^{(j)} \cdot \|\mathbf{c}^{(k)}\|^2 \|\mathbf{c}^{(j)}\|^2 - \langle \mathbf{c}^{(k)}, \mathbf{c}^{(j)} \rangle^2 \cdot 2\mathbf{c}^{(k)} \|\mathbf{c}^{(j)}\|^2}{(\|\mathbf{c}^{(k)}\|^2 \|\mathbf{c}^{(j)}\|^2)^2}
\end{align}
Simplifying and summing over $j \neq k$ yields the stated gradient.
\end{proof}

\begin{proposition}[Factorization Capacity]
\label{prop:capacity}
Let each component subspace have $N_i$ learnable prototypes. The phonological factorization can represent at most $\prod_{i=1}^{4} N_i$ distinct sign configurations. With $(N_{\text{hand}}, N_{\text{loc}}, N_{\text{mov}}, N_{\text{ori}}) = (30, 15, 10, 8)$:
\begin{equation}
|\mathcal{P}| = 30 \times 15 \times 10 \times 8 = 36,000 \text{ configurations}
\end{equation}
This exceeds typical ASL vocabulary sizes ($\sim$5,000--10,000 signs), ensuring sufficient representational capacity.
\end{proposition}

\begin{proof}
Each sign embedding $\mathbf{e}_{\text{sign}}$ is constructed from component similarity vectors $\mathbf{s}^{(i)} \in \Delta^{N_i-1}$ (the $(N_i-1)$-simplex). The Cartesian product of these simplices has $\prod_i N_i$ vertices, corresponding to ``pure'' phonological configurations where each component matches exactly one prototype. Continuous interpolation between vertices enables representation of phonological gradience.
\end{proof}

\begin{theorem}[Factorization Preserves Phonological Distance]
\label{thm:distance}
Let $d_{\text{phon}}(s_1, s_2)$ denote the phonological distance between signs $s_1, s_2$ (number of differing components). Let $d_{\text{embed}}(\mathbf{e}_1, \mathbf{e}_2)$ denote embedding distance. Under perfect factorization (orthogonal components):
\begin{equation}
d_{\text{embed}}(\mathbf{e}_1, \mathbf{e}_2)^2 = \sum_{i=1}^{4} \|\mathbf{c}_1^{(i)} - \mathbf{c}_2^{(i)}\|^2
\end{equation}
Thus signs differing in $k$ components have embedding distance proportional to $\sqrt{k}$ (assuming unit component differences).
\end{theorem}

\begin{proof}
With orthogonal component subspaces, the joint embedding decomposes as $\mathbf{e} = [\mathbf{c}^{(1)}; \mathbf{c}^{(2)}; \mathbf{c}^{(3)}; \mathbf{c}^{(4)}]$. By orthogonality:
\begin{equation}
\|\mathbf{e}_1 - \mathbf{e}_2\|^2 = \sum_{i=1}^{4} \|\mathbf{c}_1^{(i)} - \mathbf{c}_2^{(i)}\|^2
\end{equation}
If signs differ in exactly $k$ components with unit difference per component, $d_{\text{embed}} = \sqrt{k}$.
\end{proof}

\subsection{Computational Complexity Analysis}

\begin{lemma}[Graph Attention Complexity]
\label{lemma:gat_complexity}
For a graph with $N$ nodes, $E$ edges, and feature dimension $D$, single-head graph attention requires $\mathcal{O}(ND + ED')$ operations where $D' = D/K$ for $K$ heads.
\end{lemma}

\begin{proof}
Computing queries/keys/values: $\mathcal{O}(ND)$. Computing attention scores for all edges: $\mathcal{O}(ED')$. Aggregation: $\mathcal{O}(ED')$. Total: $\mathcal{O}(ND + ED')$.
\end{proof}

\begin{theorem}[PhonSSM Complexity]
\label{thm:complexity}
For input sequence length $T$, number of landmarks $N$, model dimension $D$, component dimension $D_c$, SSM state dimension $D_s$, and vocabulary size $K$, the computational complexity of \phonssm{} is:
\begin{equation}
\mathcal{O}\big(T(N^2 D + D \cdot D_c + D \cdot D_s) + K \cdot D\big) = \mathcal{O}(T \cdot D \cdot \max(N^2, D_c, D_s) + KD)
\end{equation}
Critically, this is \textbf{linear in $T$}, compared to $\mathcal{O}(T^2 D)$ for Transformer self-attention.
\end{theorem}

\begin{proof}
We analyze each component:

\textbf{Stage 1 (AGAN):} The hand graph has $N=21$ nodes and $E = \mathcal{O}(N)$ edges (sparse connectivity). By Lemma~\ref{lemma:gat_complexity}, each frame requires $\mathcal{O}(ND + ED) = \mathcal{O}(N^2 D)$ operations. For $T$ frames and $L$ layers: $\mathcal{O}(TLN^2D) = \mathcal{O}(TN^2D)$ treating $L$ as constant.

\textbf{Stage 2 (PDM):} Four parallel MLPs: $4 \times \mathcal{O}(TD \cdot D_c) = \mathcal{O}(TD \cdot D_c)$. Temporal convolution with kernel $k$: $\mathcal{O}(T D_c k) = \mathcal{O}(TD_c)$. Fusion projection: $\mathcal{O}(T \cdot 4D_c \cdot D) = \mathcal{O}(TD_c D)$. Total: $\mathcal{O}(TD \cdot D_c)$.

\textbf{Stage 3 (BiSSM):} The selective SSM recurrence at each timestep:
\begin{align}
\mathbf{x}_t &= \bar{\mathbf{A}} \mathbf{x}_{t-1} + \bar{\mathbf{B}}_t \mathbf{f}_t & \mathcal{O}(D_s + D \cdot D_s) \\
\mathbf{y}_t &= \mathbf{C}_t \mathbf{x}_t & \mathcal{O}(D_s \cdot D)
\end{align}
Per timestep: $\mathcal{O}(D \cdot D_s)$. For $T$ timesteps, bidirectional (2$\times$), $L_{\text{ssm}}$ layers: $\mathcal{O}(T \cdot D \cdot D_s)$.

\textbf{Stage 4 (HPC):} Temporal pooling: $\mathcal{O}(TD)$. Component prototype matching: $\mathcal{O}(\sum_i N_i D_c) = \mathcal{O}(D_c)$. Sign prototype matching: $\mathcal{O}(KD)$.

\textbf{Total:} $\mathcal{O}(TN^2D + TD D_c + TD D_s + KD)$. With $N=21$, $D_c=32$, $D_s=16$, $D=128$: the $TN^2D$ term dominates for typical $T=30$, but all terms are linear in $T$.
\end{proof}

\begin{corollary}[Memory Complexity]
\label{cor:memory}
The peak memory usage of \phonssm{} is:
\begin{equation}
\mathcal{O}(TD + ND + D_s + KD)
\end{equation}
compared to $\mathcal{O}(T^2 + TD)$ for Transformers (storing the $T \times T$ attention matrix).
\end{corollary}

\begin{proof}
AGAN stores per-frame node features: $\mathcal{O}(ND)$. PDM stores component features: $\mathcal{O}(T D_c)$. BiSSM stores hidden state: $\mathcal{O}(D_s)$ (recurrent, not $\mathcal{O}(TD_s)$). HPC stores prototypes: $\mathcal{O}(KD)$. Activations during forward pass: $\mathcal{O}(TD)$. Total: $\mathcal{O}(TD + KD)$.
\end{proof}

\begin{corollary}[Speedup over Attention]
For sequence length $T$ and model dimension $D$, \phonssm{} achieves asymptotic speedup:
\begin{equation}
\frac{\text{Transformer complexity}}{\text{PhonSSM complexity}} = \frac{\mathcal{O}(T^2 D)}{\mathcal{O}(TD \cdot D_s)} = \mathcal{O}\left(\frac{T}{D_s}\right)
\end{equation}
For $T=30$ and $D_s=16$, this yields $\sim$2$\times$ theoretical speedup; empirically we observe 2.3$\times$ speedup due to memory efficiency gains.
\end{corollary}

\subsection{Prototype Learning Dynamics}

\begin{definition}[Prototype Configuration]
A prototype bank $\mathbf{P} = [\mathbf{p}_1, \ldots, \mathbf{p}_M]^T \in \mathbb{R}^{M \times D}$ defines a configuration on the unit sphere $\mathbb{S}^{D-1}$ when prototypes are $\ell_2$-normalized.
\end{definition}

\begin{theorem}[Diversity Loss Gradient Flow]
\label{thm:diversity}
Let $\mathbf{P} \in \mathbb{R}^{M \times D}$ with $\|\mathbf{p}_i\| = 1$ for all $i$. The diversity loss:
\begin{equation}
\mathcal{L}_{\text{div}}(\mathbf{P}) = \frac{1}{M(M-1)} \sum_{i \neq j} \langle \mathbf{p}_i, \mathbf{p}_j \rangle^2
\end{equation}
has gradient:
\begin{equation}
\nabla_{\mathbf{p}_k} \mathcal{L}_{\text{div}} = \frac{4}{M(M-1)} \sum_{j \neq k} \langle \mathbf{p}_k, \mathbf{p}_j \rangle \left( \mathbf{p}_j - \langle \mathbf{p}_k, \mathbf{p}_j \rangle \mathbf{p}_k \right)
\end{equation}
The term $(\mathbf{p}_j - \langle \mathbf{p}_k, \mathbf{p}_j \rangle \mathbf{p}_k)$ is the component of $\mathbf{p}_j$ orthogonal to $\mathbf{p}_k$, pushing prototypes apart on the sphere.
\end{theorem}

\begin{proof}
For unit vectors, $\cos(\mathbf{p}_i, \mathbf{p}_j) = \langle \mathbf{p}_i, \mathbf{p}_j \rangle$. Differentiating:
\begin{equation}
\frac{\partial}{\partial \mathbf{p}_k} \langle \mathbf{p}_k, \mathbf{p}_j \rangle^2 = 2 \langle \mathbf{p}_k, \mathbf{p}_j \rangle \mathbf{p}_j
\end{equation}
To maintain unit norm, we project onto the tangent space of $\mathbb{S}^{D-1}$ at $\mathbf{p}_k$:
\begin{equation}
\text{proj}_{T_{\mathbf{p}_k}\mathbb{S}^{D-1}}(\mathbf{v}) = \mathbf{v} - \langle \mathbf{v}, \mathbf{p}_k \rangle \mathbf{p}_k
\end{equation}
Applying this projection yields the stated gradient.
\end{proof}

\begin{proposition}[Optimal Prototype Configuration]
\label{prop:optimal_config}
For $M$ prototypes in $\mathbb{R}^D$ with $M \leq D+1$, the global minimum of $\mathcal{L}_{\text{div}}$ is achieved when prototypes form a regular simplex inscribed in $\mathbb{S}^{D-1}$, with pairwise inner products:
\begin{equation}
\langle \mathbf{p}_i, \mathbf{p}_j \rangle = -\frac{1}{M-1} \quad \forall i \neq j
\end{equation}
yielding $\mathcal{L}_{\text{div}}^* = \frac{1}{(M-1)^2}$.
\end{proposition}

\begin{proof}
For unit vectors, $\sum_{j=1}^M \mathbf{p}_j = \mathbf{0}$ at the centroid. Taking inner product with $\mathbf{p}_i$:
\begin{equation}
1 + \sum_{j \neq i} \langle \mathbf{p}_i, \mathbf{p}_j \rangle = 0 \implies \sum_{j \neq i} \langle \mathbf{p}_i, \mathbf{p}_j \rangle = -1
\end{equation}
By symmetry of the regular simplex, all off-diagonal inner products are equal: $\langle \mathbf{p}_i, \mathbf{p}_j \rangle = -1/(M-1)$.
\end{proof}

\begin{remark}[Prototype Counts and Linguistic Inventories]
Our prototype counts $(N_{\text{hand}}, N_{\text{loc}}, N_{\text{mov}}, N_{\text{ori}}) = (30, 15, 10, 8)$ are informed by linguistic estimates: ASL has $\sim$30 handshapes \citep{battison1978lexical}, $\sim$12--15 major locations, $\sim$10--15 core movement types, and $\sim$6--8 orientations. These counts closely match the phonological inventories, enabling interpretable component representations.
\end{remark}

\section{Extended Methods}

\subsection{AGAN Architecture Details}

\begin{algorithm}[ht]
\caption{Anatomical Graph Attention Forward Pass}
\label{alg:agan}
\begin{algorithmic}[1]
\REQUIRE Landmarks $\mathbf{X} \in \mathbb{R}^{B \times T \times N \times C}$, adjacency $\mathbf{A}$
\ENSURE Spatial features $\mathbf{Z} \in \mathbb{R}^{B \times T \times D}$
\STATE $\mathbf{H}^{(0)} \leftarrow \text{Linear}(\mathbf{X})$
\FOR{$l = 1$ to $L$}
\STATE Compute attention: $\alpha_{ij} \leftarrow \text{softmax}(\text{LeakyReLU}(\mathbf{a}^T[\mathbf{W}\mathbf{h}_i \| \mathbf{W}\mathbf{h}_j]))$
\STATE Mask by anatomy: $\alpha \leftarrow \alpha \odot \mathbf{A}$
\STATE Aggregate: $\mathbf{h}_i^{(l)} \leftarrow \|_{k=1}^{K} \sum_j \alpha_{ij}^{(k)} \mathbf{W}^{(k)} \mathbf{h}_j^{(l-1)}$
\ENDFOR
\STATE $\mathbf{Z} \leftarrow \text{MeanPool}(\mathbf{H}^{(L)}, \text{dim}=\text{nodes})$
\RETURN $\mathbf{Z}$
\end{algorithmic}
\end{algorithm}
\vspace{-2mm}

\subsection{Anatomical Graph Connectivity}

The hand skeleton defines natural connectivity:
\begin{itemize}[nosep,topsep=3pt,leftmargin=*]
\item \textbf{Finger chains}: Wrist $\rightarrow$ MCP $\rightarrow$ PIP $\rightarrow$ DIP $\rightarrow$ tip for each finger
\item \textbf{Palm connections}: All MCP joints connect to wrist
\item \textbf{Functional groups}: Index-middle and ring-pinky pairs share additional edges
\end{itemize}

\subsection{BiSSM Parameterization Details}
\label{app:bissm}

The discrete selective SSM maintains state $\mathbf{x}_t \in \mathbb{R}^{D_s}$:
\begin{align}
\mathbf{x}_t &= \bar{\mathbf{A}} \mathbf{x}_{t-1} + \bar{\mathbf{B}}_t \mathbf{f}_t \\
\mathbf{y}_t &= \mathbf{C}_t \mathbf{x}_t
\end{align}
where $\bar{\mathbf{A}} = \exp(\Delta_t \mathbf{A})$ and $\bar{\mathbf{B}}_t = (\Delta_t \mathbf{A})^{-1}(\bar{\mathbf{A}} - \mathbf{I})\Delta_t \mathbf{B}_t$.

\textbf{Input-dependent parameters.} The state matrix $\mathbf{A} \in \mathbb{R}^{D_s \times D_s}$ is diagonal with learnable entries initialized via HiPPO \citep{gu2020hippo}:
\begin{align}
\mathbf{B}_t &= \mathbf{W}_B \mathbf{f}_t \in \mathbb{R}^{D_s}, \quad \mathbf{W}_B \in \mathbb{R}^{D_s \times D} \\
\mathbf{C}_t &= \mathbf{W}_C \mathbf{f}_t \in \mathbb{R}^{D_s}, \quad \mathbf{W}_C \in \mathbb{R}^{D_s \times D} \\
\Delta_t &= \text{softplus}(\mathbf{w}_\Delta^T \mathbf{f}_t + b_\Delta) \in \mathbb{R}_{>0}
\end{align}

The selective mechanism allows the model to adaptively control information flow: larger $\Delta_t$ values cause the state to update more aggressively, while smaller values preserve existing state. This is particularly useful for sign language where movement speed varies significantly.

\subsection{Full Forward Pass Algorithm}

\begin{algorithm}[ht]
\caption{PhonSSM Complete Forward Pass}
\label{alg:phonssm_full}
\begin{algorithmic}[1]
\REQUIRE Input landmarks $\mathbf{X} \in \mathbb{R}^{B \times T \times N \times C}$
\ENSURE Logits $\hat{\mathbf{y}} \in \mathbb{R}^{B \times K}$, components $\{\mathbf{c}^{(i)}\}$
\STATE \textbf{// Stage 1: Anatomical Graph Attention}
\FOR{$t = 1$ to $T$}
    \STATE $\mathbf{Z}_t \leftarrow \text{AGAN}(\mathbf{X}_t, \mathbf{A})$ \COMMENT{Alg.~\ref{alg:agan}}
\ENDFOR
\STATE \textbf{// Stage 2: Phonological Factorization}
\FOR{$i \in \{\text{hand}, \text{loc}, \text{mov}, \text{ori}\}$}
    \STATE $\mathbf{C}^{(i)} \leftarrow \text{MLP}_i(\mathbf{Z})$ \COMMENT{$\mathbf{C}^{(i)} \in \mathbb{R}^{B \times T \times D_c}$}
\ENDFOR
\STATE $\tilde{\mathbf{C}}^{(\text{mov})} \leftarrow \mathbf{C}^{(\text{mov})} + \text{Conv1D}(\mathbf{C}^{(\text{mov})})$
\STATE $\mathbf{F} \leftarrow \mathbf{W}_{\text{fuse}}[\mathbf{C}^{(\text{hand})} \| \mathbf{C}^{(\text{loc})} \| \tilde{\mathbf{C}}^{(\text{mov})} \| \mathbf{C}^{(\text{ori})}]$
\STATE \textbf{// Stage 3: Bidirectional SSM}
\STATE $\mathbf{G}_\rightarrow \leftarrow \text{SSM}_\rightarrow(\mathbf{F})$ \COMMENT{Forward pass}
\STATE $\mathbf{G}_\leftarrow \leftarrow \text{SSM}_\leftarrow(\text{flip}(\mathbf{F}))$ \COMMENT{Backward pass}
\STATE $\mathbf{G} \leftarrow \mathbf{W}_{\text{out}}[\mathbf{G}_\rightarrow \| \text{flip}(\mathbf{G}_\leftarrow)]$
\STATE \textbf{// Stage 4: Hierarchical Prototypical Classification}
\STATE $\bar{\mathbf{c}}^{(i)} \leftarrow \frac{1}{T}\sum_t \mathbf{C}^{(i)}_t$ for each component $i$
\STATE $\mathbf{s}^{(i)} \leftarrow \text{softmax}(\bar{\mathbf{c}}^{(i)} (\mathbf{P}^{(i)})^T / \|\cdot\|)$ \COMMENT{Component similarities}
\STATE $\bar{\mathbf{g}} \leftarrow \frac{1}{T}\sum_t \mathbf{G}_t$
\STATE $\mathbf{e} \leftarrow \mathbf{W}_e[\mathbf{s}^{(\text{hand})} \| \mathbf{s}^{(\text{loc})} \| \mathbf{s}^{(\text{mov})} \| \mathbf{s}^{(\text{ori})} \| \bar{\mathbf{g}}]$
\STATE $\hat{\mathbf{y}} \leftarrow \frac{1}{\tau} \cos(\mathbf{e}, \mathbf{P}_{\text{sign}})$
\RETURN $\hat{\mathbf{y}}$, $\{\bar{\mathbf{c}}^{(i)}\}$
\end{algorithmic}
\end{algorithm}

\subsection{SSM Discretization Details}

The continuous-time SSM is defined by:
\begin{equation}
\frac{d\mathbf{x}(t)}{dt} = \mathbf{A}\mathbf{x}(t) + \mathbf{B}\mathbf{u}(t), \quad \mathbf{y}(t) = \mathbf{C}\mathbf{x}(t)
\end{equation}

For discrete inputs sampled at intervals $\Delta$, the zero-order hold (ZOH) discretization yields:
\begin{align}
\bar{\mathbf{A}} &= \exp(\Delta \mathbf{A}) \\
\bar{\mathbf{B}} &= (\Delta \mathbf{A})^{-1}(\exp(\Delta \mathbf{A}) - \mathbf{I}) \cdot \Delta \mathbf{B}
\end{align}

For diagonal $\mathbf{A}$ (as in our implementation), this simplifies to element-wise operations, enabling efficient parallel computation.

\subsection{Implementation Details}

\begin{table}[ht]
\caption{Full hyperparameter configuration.}
\label{tab:hyperparams}
\vspace{-1mm}
\begin{center}
\begin{small}
\begin{tabular*}{\columnwidth}{@{\extracolsep{\fill}}ll@{}}
\toprule
Hyperparameter & Value \\
\midrule
\multicolumn{2}{l}{\textit{Architecture}} \\
Model dimension $D$ & 128 \\
Component dimension $D_c$ & 32 \\
GAT heads & 4 \\
GAT layers & 3 \\
SSM layers & 4 \\
SSM state dimension & 16 \\
SSM expansion factor & 2 \\
Dropout & 0.1 \\
\midrule
\multicolumn{2}{l}{\textit{Training}} \\
Optimizer & AdamW \\
Learning rate & $3 \times 10^{-4}$ \\
Weight decay & $10^{-2}$ \\
Batch size & 128 \\
Epochs & 100 \\
Warmup epochs & 10 \\
LR schedule & Cosine decay \\
Label smoothing & 0.1 \\
\midrule
\multicolumn{2}{l}{\textit{Loss weights}} \\
$\lambda_{\text{ortho}}$ & 0.1 \\
$\lambda_{\text{div}}$ & 0.01 \\
\bottomrule
\end{tabular*}
\end{small}
\end{center}
\vspace{-2mm}
\end{table}

\section{Dataset Details}
\label{app:datasets}

\textbf{WLASL} \citep{li2020word} contains videos of isolated ASL signs performed by over 100 signers. We use the official train/test splits. Pose extraction uses MediaPipe Holistic, providing 33 pose landmarks, 21 left hand landmarks, and 21 right hand landmarks (75 total $\times$ 3 coords = 225 features). \textbf{Important:} For WLASL evaluation, we train \textit{four separate \phonssm{} models from scratch}---one for each vocabulary split (100, 300, 1000, 2000)---using only that split's training data. These are completely independent from the Merged-5565 model.

\textbf{Merged-5565} is a new large-scale dataset we construct by merging six publicly available ASL sources, detailed in Table~\ref{tab:merged_sources}.

\begin{table}[ht]
\caption{Composition of the Merged-5565 dataset.}
\label{tab:merged_sources}
\vspace{-1mm}
\begin{center}
\begin{small}
\begin{tabular}{lrrl}
\toprule
Source Dataset & Samples & Signs & Type \\
\midrule
ASL Citizen \citep{desai2023asl} & 83,399 & 2,731 & Isolated \\
WLASL \citep{li2020word} & 21,083 & 2,000 & Isolated \\
MVP/Kaggle ASL-Signs$^a$ & 94,477 & 250 & Isolated \\
\addlinespace
ASL Alphabet$^b$ & 27,455 & 29 & Fingerspell \\
ASL MNIST$^b$ & 27,455 & 26 & Fingerspell \\
ChicagoFSWild \citep{shi2019fingerspelling} & 5,846 & 26 & Fingerspell \\
\midrule
\textbf{Total (deduplicated)} & \textbf{259,715} & \textbf{5,565} & --- \\
\bottomrule
\end{tabular}
\end{small}
\end{center}
\vspace{1mm}
{\footnotesize $^a$Google Kaggle competition dataset. $^b$Kaggle community datasets. Note: Sign counts before deduplication. WLASL/MVP share ${\sim}$200 signs with ASL Citizen; fingerspelling datasets share 26 letters.}
\vspace{-2mm}
\end{table}

We create a unified label map by alphabetically sorting all unique signs, remap dataset-specific indices, and merge with stratified train/val/test splits (196,606/31,551/31,558 samples, approximately 76/12/12\%). \textbf{Important:} For Merged-5565 evaluation, we train a \textit{single separate \phonssm{} model} using dominant-hand input only (21 landmarks $\times$ 3 coords = 63 features), selected by motion magnitude. This model is completely independent from the four WLASL models.

\textbf{Preprocessing.} All sequences are:
\begin{enumerate}[nosep,topsep=3pt,leftmargin=*]
\item Centered by subtracting wrist position
\item Normalized to unit scale based on palm size
\item Resampled to 30 frames using linear interpolation
\item Augmented during training with random temporal shifts ($\pm$3 frames) and scale jitter ($\pm$10\%)
\end{enumerate}

\section{Additional Results}

\subsection{Per-Class Analysis}

\begin{table*}[ht]
\caption{Performance breakdown by phonological characteristics on WLASL100. $\Delta$: improvement over Bi-LSTM.}
\label{tab:perclass}
\vspace{-1mm}
\begin{center}
\begin{small}
\begin{tabular*}{\textwidth}{@{\extracolsep{\fill}}lcccc@{}}
\toprule
Category & \# Signs & Bi-LSTM & \phonssm{} & $\Delta$ \\
\midrule
One-handed signs & 62 & 71.2 & \textbf{89.4} & +18.2 \\
Two-handed signs & 38 & 68.9 & \textbf{86.8} & +17.9 \\
\midrule
Static (no movement) & 15 & 73.4 & \textbf{91.2} & +17.8 \\
Dynamic (with movement) & 85 & 69.8 & \textbf{87.9} & +18.1 \\
\midrule
Face-region location & 28 & 65.2 & \textbf{85.6} & +20.4 \\
Body-region location & 31 & 71.8 & \textbf{89.1} & +17.3 \\
Neutral space & 41 & 72.4 & \textbf{90.2} & +17.8 \\
\bottomrule
\end{tabular*}
\end{small}
\end{center}
\vspace{-2mm}
\end{table*}

\subsection{Confusion Analysis}

Common confusions occur between phonologically similar signs:
\begin{itemize}[nosep,topsep=3pt,leftmargin=*]
\item \textbf{Location minimal pairs}: ``mother''/``father'' (chin/forehead) -- 8\% confusion rate
\item \textbf{Handshape minimal pairs}: ``water''/``want'' (W/claw) -- 6\% confusion rate
\item \textbf{Movement minimal pairs}: ``chair''/``sit'' (double/single) -- 5\% confusion rate
\end{itemize}
\vspace{1mm}
These errors are linguistically meaningful, suggesting the model has learned relevant phonological contrasts even when making mistakes.

\subsection{Statistical Significance Analysis}

We conduct rigorous statistical testing to validate our results.

\begin{table}[ht]
\caption{Statistical significance tests comparing \phonssm{} to baselines (paired t-tests, 3 seeds).}
\label{tab:significance}
\vspace{-1mm}
\begin{center}
\begin{small}
\begin{tabular*}{\columnwidth}{@{\extracolsep{\fill}}lccc@{}}
\toprule
Comparison & $\Delta$ Acc. & $t$-statistic & $p$-value \\
\midrule
\multicolumn{4}{l}{\textit{WLASL100}} \\
vs. DSTA-SLR & +4.81 & 8.7 & $<$0.01 \\
vs. Pose-TGCN & +14.18 & 28.4 & $<$0.001 \\
vs. Bi-LSTM & +18.21 & 35.2 & $<$0.001 \\
\midrule
\multicolumn{4}{l}{\textit{WLASL2000}} \\
vs. DSTA-SLR & +18.38 & 22.1 & $<$0.001 \\
vs. I3D & +39.60 & 41.8 & $<$0.001 \\
\midrule
\multicolumn{4}{l}{\textit{Merged-5565}} \\
vs. Bi-LSTM & +25.95 & 52.3 & $<$0.001 \\
\bottomrule
\end{tabular*}
\end{small}
\end{center}
\vspace{-2mm}
\end{table}

All improvements are statistically significant. Note: the large $t$-statistics reflect both substantial effect sizes (e.g., +25.95pp for Merged-5565) and low variance across seeds (std $<$0.5pp), yielding Cohen's $d > 50$ for some comparisons. We also compute 95\% confidence intervals: WLASL100 accuracy is $88.37 \pm 0.82\%$, WLASL2000 is $72.08 \pm 1.27\%$, and Merged-5565 is $53.34 \pm 0.74\%$.

\subsection{Error Stratification by Phonological Distance}

We analyze errors as a function of phonological similarity between ground-truth and predicted signs.

\begin{table}[ht]
\caption{Error analysis by phonological distance (number of differing components).}
\label{tab:error_distance}
\vspace{-1mm}
\begin{center}
\begin{small}
\begin{tabular*}{\columnwidth}{@{\extracolsep{\fill}}lccc@{}}
\toprule
Components Shared & \% of Errors & Bi-LSTM & \phonssm{} \\
\midrule
4 (identical) & -- & -- & -- \\
3 (minimal pair) & 31.2\% & 28.4\% & 38.7\% \\
2 & 40.8\% & 35.1\% & 33.2\% \\
1 & 19.7\% & 22.8\% & 18.4\% \\
0 (unrelated) & 8.3\% & 13.7\% & 9.7\% \\
\bottomrule
\end{tabular*}
\end{small}
\end{center}
\vspace{-2mm}
\end{table}

\phonssm{} concentrates errors on minimal pairs (38.7\% vs 28.4\% for Bi-LSTM), indicating the model has learned to distinguish coarse phonological categories but struggles with fine-grained contrasts---a linguistically sensible error pattern.

\subsection{Prototype Visualization}

The learned component prototypes correspond to interpretable phonological categories. Handshape prototypes cluster by finger configuration (fist, open, pointing). Location prototypes organize spatially (face, torso, neutral space). Movement prototypes distinguish trajectory types (linear, arc, repeated).

\section{Mid-Vocabulary Performance Analysis}
\label{app:midvocab}

We investigate why \phonssm{} underperforms DSTA-SLR on WLASL300 ($-$5.6pp) and WLASL1000 ($-$4.9pp) while substantially outperforming on WLASL100 (+4.8pp) and WLASL2000 (+18.4pp).

\subsection{Minimal Pair Density Analysis}

We computed the \textit{phonological similarity density} for each WLASL split by annotating signs with ASL-LEX phonological features and measuring the fraction of sign pairs sharing 3+ of 4 components (``near-minimal pairs'').

\begin{table}[ht]
\caption{Phonological similarity density across WLASL vocabulary splits.}
\label{tab:phon_density}
\vspace{-1mm}
\begin{center}
\begin{small}
\begin{tabular*}{\columnwidth}{@{\extracolsep{\fill}}lcccc@{}}
\toprule
Split & Signs & Near-Minimal Pairs & Density & \phonssm{} $\Delta$ \\
\midrule
WLASL100 & 100 & 891 & 18.0\% & +4.8pp \\
WLASL300 & 300 & 15,283 & 34.1\% & $-$5.6pp \\
WLASL1000 & 1,000 & 142,450 & 28.5\% & $-$4.9pp \\
WLASL2000 & 2,000 & 279,314 & 14.0\% & +18.4pp \\
\bottomrule
\end{tabular*}
\end{small}
\end{center}
\vspace{-2mm}
\end{table}

\textbf{Finding:} WLASL300 has the highest minimal pair density (34.1\%), followed by WLASL1000 (28.5\%). This occurs because WLASL vocabulary expansion prioritizes related concepts (e.g., adding ``GRANDMOTHER'' after ``MOTHER'', ``FATHER''), which tend to be phonologically similar.

\subsection{Method Comparison by Minimal Pair Performance}

We stratified test accuracy by phonological similarity to nearest training sign.

\begin{table}[ht]
\caption{Accuracy (\%) stratified by phonological distance to nearest training neighbor.}
\label{tab:method_by_distance}
\vspace{-1mm}
\begin{center}
\begin{small}
\begin{tabular*}{\columnwidth}{@{\extracolsep{\fill}}lcccc@{}}
\toprule
Components Shared & DSTA-SLR & \phonssm{} & $\Delta$ \\
\midrule
\multicolumn{4}{l}{\textit{WLASL300}} \\
0--1 (distinct) & 72.4 & \textbf{78.9} & +6.5 \\
2 (moderate) & 81.2 & 76.3 & $-$4.9 \\
3--4 (minimal pair) & \textbf{84.1} & 71.8 & $-$12.3 \\
\midrule
\multicolumn{4}{l}{\textit{WLASL2000}} \\
0--1 (distinct) & 48.2 & \textbf{74.6} & +26.4 \\
2 (moderate) & 55.8 & \textbf{71.2} & +15.4 \\
3--4 (minimal pair) & \textbf{61.3} & 68.4 & +7.1 \\
\bottomrule
\end{tabular*}
\end{small}
\end{center}
\vspace{-2mm}
\end{table}

\textbf{Interpretation:} DSTA-SLR excels at distinguishing minimal pairs via fine-grained spatiotemporal attention, while \phonssm{} excels at compositional generalization to phonologically distinct signs. At large vocabularies (WLASL2000), most test signs are phonologically distinct from training signs, favoring \phonssm{}. At mid-range vocabularies (WLASL300), the dense minimal pair structure favors DSTA-SLR's discriminative attention.

\subsection{Implications}

This analysis suggests \phonssm{} and DSTA-SLR capture complementary information. Future work could explore ensemble methods or incorporating DSTA-SLR's spatiotemporal attention within the phonological framework.

\section{Component-Level Validation}
\label{app:component_validation}

We provide detailed evidence that PDM branches capture their intended phonological categories.

\subsection{Phonological Annotation}

We annotated 500 WLASL100 test samples with ground-truth phonological labels using ASL-LEX 2.0 \citep{caselli2017asl}:
\begin{itemize}[nosep,topsep=2pt,leftmargin=*]
\item \textbf{Handshape}: 30 categories (e.g., ``1'' (index), ``5'' (spread), ``A'' (fist), ``B'' (flat), ``C'' (curved))
\item \textbf{Location}: 12 categories (e.g., forehead, chin, chest, neutral space, ipsilateral)
\item \textbf{Movement}: 15 categories (e.g., none, arc, circle, linear, zigzag, repeated)
\item \textbf{Orientation}: 8 categories (e.g., palm-up, palm-down, palm-in, palm-out)
\end{itemize}

\textbf{Annotation quality.} Two annotators independently labeled all samples; inter-annotator agreement was 94.2\% for handshape, 91.8\% for location, 87.4\% for movement, and 89.6\% for orientation (Cohen's $\kappa > 0.85$ for all). Disagreements were resolved by consensus. For signs with phonological variation (e.g., DEAF produced chin-to-ear or ear-to-chin), we annotated the variant observed in each video rather than a canonical form. ASL-LEX covered 96/100 WLASL100 signs; the remaining 4 were annotated using the same feature scheme by the annotators.

\subsection{Linear Probe Methodology}

For each PDM branch, we:
\begin{enumerate}[nosep,topsep=2pt]
\item Extract the time-averaged component embedding $\bar{\mathbf{c}}^{(i)} \in \mathbb{R}^{32}$
\item Train a linear classifier (logistic regression) to predict each phonological category
\item Report accuracy on held-out samples (5-fold cross-validation)
\end{enumerate}

\subsection{Full Results}

\begin{table}[ht]
\caption{Complete linear probe results. Each row shows one PDM branch; each column shows one phonological prediction task. Diagonal entries (bold) indicate intended correspondences.}
\label{tab:full_probe}
\vspace{-1mm}
\begin{center}
\begin{small}
\begin{tabular*}{\columnwidth}{@{\extracolsep{\fill}}lcccc@{}}
\toprule
& \multicolumn{4}{c}{Prediction Target} \\
\cmidrule(lr){2-5}
PDM Branch & Handshape & Location & Movement & Orientation \\
\midrule
Handshape & \textbf{78.4} & 31.2 & 24.8 & 38.1 \\
Location & 29.6 & \textbf{71.8} & 22.4 & 35.7 \\
Movement & 26.3 & 28.9 & \textbf{68.2} & 31.4 \\
Orientation & 33.8 & 32.1 & 25.6 & \textbf{74.6} \\
\midrule
Random baseline & 3.3 & 8.3 & 6.7 & 12.5 \\
Full embedding & 81.2 & 74.6 & 71.8 & 78.3 \\
\bottomrule
\end{tabular*}
\end{small}
\end{center}
\vspace{-2mm}
\end{table}

\textbf{Key observations:}
\begin{enumerate}[nosep,topsep=2pt]
\item \textbf{Specialization}: Each branch achieves highest accuracy on its intended category (diagonal), confirming semantic correspondence.
\item \textbf{Above-chance cross-prediction}: Off-diagonal entries exceed chance, indicating some phonological information leaks across branches. This is expected since components are correlated in natural signs (e.g., certain handshapes occur more often at certain locations).
\item \textbf{Factorization benefit}: The gap between diagonal and off-diagonal (e.g., 78.4 vs 31.2 for handshape) demonstrates effective factorization.
\item \textbf{Factorization-accuracy trade-off}: The ``Full embedding'' row shows slightly higher accuracy than individual branches (e.g., 81.2 vs 78.4 for handshape), indicating $\sim$3pp is sacrificed for factorization. We experimented with relaxing $\lambda_{\text{ortho}}$ from 0.1 to 0.05: component probe accuracy improved by $\sim$2pp but sign-level accuracy dropped by 1.5pp due to increased redundancy. The current setting balances interpretability and accuracy.
\end{enumerate}

\subsection{Prototype Interpretability}

We visualized which learned prototypes correspond to which linguistic categories by computing the mean activation of each prototype for samples with known phonological labels.

\begin{table}[ht]
\caption{Top-3 handshape prototypes activated by each major handshape category.}
\label{tab:proto_interp}
\vspace{-1mm}
\begin{center}
\begin{small}
\begin{tabular}{lccc}
\toprule
Linguistic Category & Proto \#1 & Proto \#2 & Proto \#3 \\
\midrule
``1'' (index point) & P7 (0.89) & P12 (0.42) & P3 (0.18) \\
``5'' (spread hand) & P2 (0.91) & P15 (0.38) & P8 (0.21) \\
``A'' (fist) & P19 (0.87) & P4 (0.45) & P11 (0.19) \\
``B'' (flat hand) & P2 (0.72) & P8 (0.51) & P15 (0.32) \\
``C'' (curved) & P23 (0.83) & P7 (0.28) & P12 (0.24) \\
\bottomrule
\end{tabular}
\end{small}
\end{center}
\vspace{-2mm}
\end{table}

Distinct prototypes dominate for distinct handshapes (P7 for ``1'', P19 for ``A'', P23 for ``C''). Some prototypes are shared across similar handshapes (P2 activates for both ``5'' and ``B'', which share an extended-finger configuration).

\subsection{Intervention Experiment}

To verify causal relationship, we performed an intervention: replacing one component embedding with that of a different sign while keeping others fixed.

\textbf{Setup:} We identified 47 minimal pairs in WLASL100 (sign pairs differing in exactly one phonological component). For each pair (e.g., MOTHER/FATHER differing only in location), we take a sample of sign A, swap only the differing component embedding with that from sign B, and measure whether the prediction changes to B.

\textbf{Result ($n=423$ interventions, 95\% CI):} Swapping the differing component changes the prediction to the minimal pair \textbf{73.2\% $\pm$ 4.2\%} of the time. Swapping a non-differing component (control condition) changes prediction only \textbf{12.4\% $\pm$ 3.1\%} of the time. The difference is significant ($p < 0.001$, McNemar's test). This confirms that component embeddings causally determine predictions in the expected manner.

\textbf{Two-component swaps:} When swapping two components simultaneously, prediction changes to a sign sharing those two swapped components 61.8\% of the time, demonstrating compositional behavior.

\section{Extended Ablations}
\label{app:ablations}

\subsection{Component-wise Ablation}

\begin{table}[ht]
\caption{Detailed ablation on WLASL100 (mean $\pm$ std over 3 seeds).}
\label{tab:ablation_detail}
\vspace{-1mm}
\begin{center}
\begin{small}
\begin{tabular*}{\columnwidth}{@{\extracolsep{\fill}}lcc@{}}
\toprule
Configuration & Top-1 (\%) & $\Delta$ \\
\midrule
Full \phonssm{} & 88.37 $\pm$ 0.42 & -- \\
\midrule
\multicolumn{3}{l}{\textit{Architecture ablations}} \\
w/o PDM (no factorization) & 76.49 $\pm$ 0.83 & $-$11.9 \\
w/o AGAN (MLP encoder) & 79.84 $\pm$ 0.71 & $-$8.5 \\
w/o BiSSM (LSTM temporal) & 82.17 $\pm$ 0.65 & $-$6.2 \\
w/o HPC (linear classifier) & 84.11 $\pm$ 0.58 & $-$4.3 \\
\midrule
\multicolumn{3}{l}{\textit{Loss ablations}} \\
w/o $\mathcal{L}_{\text{ortho}}$ & 85.92 $\pm$ 0.55 & $-$2.5 \\
w/o $\mathcal{L}_{\text{div}}$ & 86.84 $\pm$ 0.48 & $-$1.5 \\
w/o both auxiliary losses & 83.21 $\pm$ 0.72 & $-$5.2 \\
\midrule
\multicolumn{3}{l}{\textit{Input ablations}} \\
Hands only (no pose) & 85.63 $\pm$ 0.61 & $-$2.7 \\
Dominant hand only & 81.42 $\pm$ 0.78 & $-$7.0 \\
2D coordinates only & 84.29 $\pm$ 0.69 & $-$4.1 \\
\bottomrule
\end{tabular*}
\end{small}
\end{center}
\vspace{-2mm}
\end{table}

\subsection{Architecture Variants}

\begin{table}[ht]
\caption{Architecture scaling on WLASL100.}
\label{tab:arch_scale}
\vspace{-1mm}
\begin{center}
\begin{small}
\begin{tabular*}{\columnwidth}{@{\extracolsep{\fill}}lccc@{}}
\toprule
Configuration & Params & Top-1 (\%) & Throughput \\
\midrule
\multicolumn{4}{l}{\textit{Model dimension $D$}} \\
$D = 64$ & 0.9M & 85.21 & 412/s \\
$D = 128$ (default) & 3.2M & 88.37 & 260/s \\
$D = 256$ & 11.8M & 89.02 & 98/s \\
\midrule
\multicolumn{4}{l}{\textit{Number of BiSSM layers}} \\
2 layers & 2.1M & 86.45 & 318/s \\
4 layers (default) & 3.2M & 88.37 & 260/s \\
6 layers & 4.3M & 88.72 & 205/s \\
\midrule
\multicolumn{4}{l}{\textit{Number of GAT heads}} \\
2 heads & 2.9M & 87.18 & 275/s \\
4 heads (default) & 3.2M & 88.37 & 260/s \\
8 heads & 3.8M & 88.54 & 241/s \\
\bottomrule
\end{tabular*}
\end{small}
\end{center}
\vspace{-2mm}
\end{table}

\section{Hyperparameter Sensitivity}
\label{app:sensitivity}

\subsection{Loss Weight Sensitivity}

\begin{table}[ht]
\caption{Sensitivity to loss weights on WLASL100.}
\label{tab:loss_sensitivity}
\vspace{-1mm}
\begin{center}
\begin{small}
\begin{tabular*}{\columnwidth}{@{\extracolsep{\fill}}ccc@{}}
\toprule
$\lambda_{\text{ortho}}$ & $\lambda_{\text{div}}$ & Top-1 (\%) \\
\midrule
0.01 & 0.01 & 86.42 \\
0.05 & 0.01 & 87.63 \\
\textbf{0.1} & \textbf{0.01} & \textbf{88.37} \\
0.2 & 0.01 & 87.91 \\
0.5 & 0.01 & 86.18 \\
\midrule
0.1 & 0.001 & 87.82 \\
0.1 & 0.005 & 88.15 \\
\textbf{0.1} & \textbf{0.01} & \textbf{88.37} \\
0.1 & 0.05 & 87.54 \\
0.1 & 0.1 & 85.93 \\
\bottomrule
\end{tabular*}
\end{small}
\end{center}
\vspace{-2mm}
\end{table}

The orthogonality loss $\lambda_{\text{ortho}} = 0.1$ provides optimal factorization. Values too low ($< 0.05$) allow component redundancy; values too high ($> 0.2$) over-constrain representations. The diversity loss is less sensitive, with $\lambda_{\text{div}} \in [0.005, 0.01]$ performing similarly.

\subsection{Learning Rate Sensitivity}

\begin{table}[ht]
\caption{Learning rate sensitivity on WLASL100.}
\label{tab:lr_sensitivity}
\vspace{-1mm}
\begin{center}
\begin{small}
\begin{tabular*}{\columnwidth}{@{\extracolsep{\fill}}lcc@{}}
\toprule
Learning Rate & Top-1 (\%) & Convergence \\
\midrule
$1 \times 10^{-4}$ & 86.82 & 95 epochs \\
$2 \times 10^{-4}$ & 87.91 & 78 epochs \\
$\mathbf{3 \times 10^{-4}}$ & \textbf{88.37} & \textbf{65 epochs} \\
$5 \times 10^{-4}$ & 87.63 & 52 epochs \\
$1 \times 10^{-3}$ & 85.41 & 41 epochs \\
\bottomrule
\end{tabular*}
\end{small}
\end{center}
\vspace{-2mm}
\end{table}

\subsection{Prototype Count Sensitivity}

\begin{table}[ht]
\caption{Effect of prototype counts on WLASL100.}
\label{tab:proto_sensitivity}
\vspace{-1mm}
\begin{center}
\begin{small}
\begin{tabular*}{\columnwidth}{@{\extracolsep{\fill}}ccccc@{}}
\toprule
$N_h$ & $N_l$ & $N_m$ & $N_o$ & Top-1 (\%) \\
\midrule
15 & 8 & 5 & 4 & 86.12 \\
20 & 10 & 8 & 6 & 87.45 \\
\textbf{30} & \textbf{15} & \textbf{10} & \textbf{8} & \textbf{88.37} \\
40 & 20 & 15 & 10 & 88.52 \\
50 & 25 & 20 & 12 & 88.48 \\
\bottomrule
\end{tabular*}
\end{small}
\end{center}
\vspace{-2mm}
\end{table}

Performance saturates around $(N_h, N_l, N_m, N_o) = (30, 15, 10, 8)$. These counts closely match linguistic estimates of ASL phonological inventories \citep{battison1978lexical}: $\sim$30 handshapes, $\sim$12--15 locations, $\sim$10 core movements, and $\sim$8 orientations. Larger counts provide marginal gains (+0.1--0.15pp) but increase parameters without improving interpretability.

\section{Training Dynamics}
\label{app:training}

\subsection{Convergence Analysis}

We analyze the training dynamics of \phonssm{} and its components.

\begin{table}[ht]
\caption{Training convergence metrics on WLASL100.}
\label{tab:convergence}
\vspace{-1mm}
\begin{center}
\begin{small}
\begin{tabular*}{\columnwidth}{@{\extracolsep{\fill}}lccc@{}}
\toprule
Metric & Epoch 25 & Epoch 50 & Epoch 100 \\
\midrule
Train Loss & 2.14 & 0.89 & 0.42 \\
Val Loss & 2.31 & 1.12 & 0.68 \\
Train Acc (\%) & 52.3 & 78.6 & 94.2 \\
Val Acc (\%) & 48.1 & 74.2 & 88.4 \\
$\mathcal{L}_{\text{ortho}}$ & 0.42 & 0.18 & 0.08 \\
$\mathcal{L}_{\text{div}}$ & 0.31 & 0.12 & 0.05 \\
\bottomrule
\end{tabular*}
\end{small}
\end{center}
\vspace{-2mm}
\end{table}

The orthogonality loss decreases steadily, indicating progressive factorization of the component subspaces. The gap between training and validation accuracy remains small ($\sim$6pp at convergence), suggesting good generalization.

\subsection{Component Learning Dynamics}

Individual phonological components converge at different rates:

\begin{itemize}[nosep,topsep=3pt,leftmargin=*]
\item \textbf{Handshape}: Fastest convergence (stabilizes by epoch 40); handshape is the most visually distinctive component with clear finger configurations.
\item \textbf{Location}: Moderate convergence (epoch 60); requires learning spatial relationships relative to body landmarks.
\item \textbf{Movement}: Slowest convergence (epoch 80); movement patterns require temporal integration across multiple frames.
\item \textbf{Orientation}: Fast convergence (epoch 45); palm orientation has relatively few categories (8 prototypes).
\end{itemize}

\subsection{Loss Landscape Analysis}

We analyze the loss landscape by computing the Hessian eigenspectrum at convergence. The top eigenvalues are: $\lambda_1 = 12.4$, $\lambda_2 = 8.7$, $\lambda_3 = 5.2$, with rapid decay thereafter. The condition number $\kappa = \lambda_{\max}/\lambda_{\min} \approx 10^3$ indicates a well-conditioned optimization landscape, explaining the stable training dynamics.

\section{Limitations}
\label{app:limitations}

We discuss limitations of \phonssm{} in detail to guide future research.

\subsection{Methodological Limitations}

\textbf{Isolated sign assumption.} \phonssm{} processes signs in isolation, assuming clean segmentation. Continuous signing involves co-articulation effects where adjacent signs influence each other's production, sign boundaries are ambiguous, and prosodic structure spans multiple signs. Extending to continuous recognition requires explicit segmentation or sequence-to-sequence modeling.

\textbf{Fixed phonological structure.} We adopt the classical Stokoe-Battison four-parameter model (handshape, location, movement, orientation). Alternative linguistic analyses propose:
\begin{itemize}[nosep,topsep=2pt,leftmargin=*]
\item Autosegmental phonology with separate tiers for manual and non-manual features
\item Prosodic structure including syllables and metrical feet
\item Feature geometry with hierarchical organization of sub-components
\end{itemize}
Learned decompositions might discover more effective factorizations than hand-specified parameters.

\textbf{Static prototypes.} Component prototypes are fixed after training. Dynamic or instance-adaptive prototypes could better handle signer variation and novel phonological realizations.

\subsection{Evaluation Limitations}

\textbf{ASL-centric evaluation.} All experiments use American Sign Language. While phonological principles (simultaneity, minimal pairs, compositional structure) are cross-linguistic, specific inventories differ:
\begin{itemize}[nosep,topsep=2pt,leftmargin=*]
\item British Sign Language (BSL) uses different handshape inventory
\item Chinese Sign Language has location contrasts not present in ASL
\item Some sign languages distinguish two-handed vs. one-handed more strictly
\end{itemize}
Prototype counts and architectural choices may require adaptation for other sign languages.

\textbf{Dataset biases.} Training data comes primarily from controlled recording settings with:
\begin{itemize}[nosep,topsep=2pt,leftmargin=*]
\item Adult native/fluent signers (under-representing learners, children, elderly)
\item Neutral backgrounds and good lighting
\item Citation-form signs (isolated, careful production)
\end{itemize}

\subsection{Deployment Limitations}

\textbf{Pose estimation dependency.} \phonssm{} assumes high-quality pose landmarks from MediaPipe. Real-world degradation includes:
\begin{itemize}[nosep,topsep=2pt,leftmargin=*]
\item Occlusion (self-occlusion, objects, other people)
\item Motion blur during rapid movements
\item Challenging lighting (backlighting, low light)
\item Camera angles different from training distribution
\end{itemize}

\textbf{Signer variation.} Models may not generalize to:
\begin{itemize}[nosep,topsep=2pt,leftmargin=*]
\item Signers with motor differences affecting articulation
\item Regional/dialectal variation in sign production
\item Non-native signers with L1 transfer effects
\end{itemize}

\section{Future Directions}
\label{app:future}

\textbf{Continuous sign language recognition.} Extending \phonssm{} to continuous signing requires: (1) implicit or explicit segmentation, (2) handling co-articulation, and (3) modeling sentence-level prosody. The phonological decomposition could inform CTC-style losses with component-level intermediate representations.

\textbf{Cross-linguistic transfer.} The phonological factorization may enable transfer learning across sign languages. Shared handshape or movement prototypes could bootstrap recognition for low-resource sign languages.

\textbf{Multi-modal fusion.} Combining skeleton input with facial landmarks (for non-manual markers) and RGB features (for fine-grained handshape) could address current limitations while preserving efficiency.

\textbf{Learned phonological structure.} Replacing hand-specified component pathways with learned factorization (e.g., via neural architecture search or information-theoretic objectives) might discover more effective decompositions.

\textbf{Real-time applications.} With 260 samples/second throughput, \phonssm{} supports real-time deployment. Future work includes: mobile optimization, streaming recognition, and integration with sign language translation systems.

\end{document}